%% file: main.tex
\documentclass{article} % For LaTeX2e
\usepackage{iclr2027_conference,times}

\input{math_commands.tex}

\usepackage{hyperref}
\usepackage{url}
\usepackage{comment}
\usepackage{booktabs} 
\usepackage[table]{xcolor}
\usepackage{graphicx}
\usepackage{wrapfig}
\usepackage{subcaption}

\usepackage{amsmath,amsthm,amsfonts,thmtools}
\usepackage{amssymb}

\newtheorem{theorem}{Theorem}[section]

\theoremstyle{definition}

\newcommand{\fref}[1] {Figure~\ref{#1}}
\newcommand{\tref}[1] {Table~\ref{#1}}

\title{Looped Actor: Depth-Recurrent Reasoning Models for Reinforcement Learning}

\author{T. Konstantin Rusch\\
ELLIS Institute Tübingen \&\\
Max Planck Institute for Intelligent Systems \&\\
Tübingen AI Center \&\\
Liquid AI\\
\texttt{tkrusch@tue.ellis.eu} \\
\And
Tim Seyde \\
Liquid AI\\
\And
Jared Boyer \\
CSAIL, MIT\\
\And
Zach J. Patterson \\
Case Western Reserve University\\
\And
Daniela Rus \\
CSAIL, MIT\\
}

\iclrfinalcopy % Uncomment for camera-ready version, but NOT for submission.
\begin{document}

\maketitle

\begin{abstract}
Looped reasoning models repeatedly apply a shared set of parameters, enabling more computation without increasing the model size. These models also support input-dependent computation by dynamically deciding when to stop looping. Motivated by the recent success of looped transformers in language modeling and reasoning, we investigate whether dynamic looping can similarly benefit sequential decision-making. We provide a complexity-theoretic motivation for this approach by showing that there exist Markov decision processes in which a state-adaptive policy achieves the optimal return with asymptotically less expected computation than any optimal fixed-runtime policy. To learn compute-adaptive policies in practice, we introduce Looped Actor, a transformer-based policy that repeatedly refines a latent representation toward a fixed point using a shared computational block. This allows the model to allocate computation adaptively by varying the number of loops based on the current state. We evaluate Looped Actor on 22 tasks across six environments, ranging from combinatorial puzzles to robotic manipulation and spanning online and offline reinforcement learning (RL) with discrete and continuous actions. Looped Actor matches or exceeds the performance of an untied baseline with 16$\times$ more parameters, with the largest gains in environments where action selection requires substantial multistep planning. For the Boxoban environment, we find that the computation allocation is structured: the number of loops increases with the number of remaining pushes and future optimal pushes become increasingly predictable from the latent state over successive loops. Together, these results highlight actor looping as a simple and efficient way to equip RL agents with adaptive computation and improve their planning capabilities. Code is available at \href{https://github.com/camail-official/LoopedActor}{\textbf{https://github.com/camail-official/LoopedActor}}.
\end{abstract}

\section{Introduction}
Sequential decision-making involves choices of varying difficulty: some require little effort, while others call for substantial computation. In a game, for example, choosing a strategy may require extensive planning, whereas executing it may involve a sequence of straightforward moves. A policy should therefore be able to devote more computation to difficult decisions and less to routine ones. 

We theoretically motivate the computational advantage of adapting the computation of the policy to a given state using an established framework in complexity theory. More concretely, we construct a family of Markov Decision Processes (MDPs) for which we show there exist state-adaptive policies with a computational advantage over any fixed-runtime policies that grows without bound along the constructed MDP family, where both achieve optimal return. In order to learn policies with adaptive computation in practice, we propose to use looped reasoning models in the context of Reinforcement Learning (RL). Looped reasoning models naturally support such adaptation by repeatedly applying the same parameterized block and halting when reasoning is complete, enabling adaptive computation without increasing model size. 

Beyond their capacity for adaptive computation, looped models have attracted growing interest as an alternative to explicit Chain-of-Thought (CoT) reasoning \citep{dehghani2018universal,bansal2022end,saunshi2025reasoning,hao2024training,wei2022chain}. Early work on looped transformers focused on algorithmic tasks that decompose into subroutines and benefit from a recurrent inductive bias \citep{dehghani2018universal,giannou2023looped,kohli2026loop,yang2311looped,fan2025looped}. More recently, looping has shown promise for reasoning in Large Language Models (LLMs) \citep{geiping2026scaling,zhu2025scaling,kapl2026growing}. In parallel, small looped models have achieved strong performance on challenging reasoning tasks with fewer than $0.01\%$ of the parameters of frontier models \citep{wang2025hierarchical,jolicoeur2025less,movahedi2026fixed,huang2026equilibrium}. These results raise a natural question: can looped models similarly provide useful reasoning capabilities together with adaptive computation for sequential decision-making?

We introduce Looped Actor, a deep reinforcement learning (RL) policy that adapts the Fixed-Point Reasoning Model (FPRM) \citep{movahedi2026fixed} to action selection. The actor repeatedly refines a latent representation using a shared transformer block and halts when its predicted action distribution stabilizes. This allows the policy to reason and adapt its computation to the current state without increasing the model’s parameter count.

Our experiments across 22 tasks in six environments show that Looped Actor combines strong performance with computational and parameter efficiency. On Boxoban, it achieves roughly twice the success rate of an untied transformer with 16 times as many parameters, while matching this baseline’s average performance across 20 offline goal-conditioned control tasks. Beyond these performance gains, we investigate how repeated computation supports planning. Probes of the actor’s latent states show that future optimal box pushes become increasingly predictable over successive loops, providing evidence that the model develops representations of multistep solutions. Its allocation of computation also reflects its ability to solve a puzzle: in successful episodes, the number of loops increases with the difficulty of a state measured via the number of remaining pushes. These results suggest that Looped Actor provides a simple mechanism for combining parameter-efficient reasoning and state-dependent computation in reinforcement learning.

\section{Related Work}
\subsection{Looped Reasoning Models}
Looped architectures can increase computation without increasing the model size by repeatedly applying the same block of parameters. The first looped transformer architecture, the Universal Transformer \citep{dehghani2018universal}, was specifically designed to leverage this recurrent inductive bias to effectively solve tasks that can be decomposed into smaller computational subroutines. This approach has since been extended to various algorithmic and arithmetic tasks \citep{giannou2023looped,kohli2026loop,yang2311looped,fan2025looped}. More recently, looped transformers have attracted growing interest for efficient reasoning in LLMs \citep{geiping2026scaling,zhu2025scaling,saunshi2025reasoning,kapl2026growing}, with looping enabling a new axis of test-time scaling.

A complementary line of work explores small looped transformers for challenging reasoning tasks. For example, Hierarchical Reasoning Models (HRMs) \citep{wang2025hierarchical} and Tiny Recursive Models (TRMs) \citep{jolicoeur2025less} have been shown to outperform frontier reasoning models on benchmarks such as ARC-AGI \citep{chollet2024arc}, using efficient hierarchical looped architectures with fewer than 0.01\% of the parameters. These results have motivated further work on efficient reasoning with looped transformers. One framework within this line of work is attractor-based reasoning, in which a looped transformer converges to an equilibrium state that represents the solution. The Fixed-Point Reasoning Model (FPRM) \citep{movahedi2026fixed} implements this approach through fixed-point dynamics. Equilibrium Reasoners (EqR) \citep{huang2026equilibrium} and Attractor Models \citep{fein2026solve} learn input-dependent attractors and use the corresponding equilibrium states as model outputs.

\subsection{Iterative Computation in RL Policies}
Additional computation per decision can support planning without explicit search such as Monte Carlo Tree Search \citep{schrittwieser2020mastering}. Early approaches embed algorithmic structure into the network, including differentiable value iteration \citep{tamar2016value} and latent tree expansion with differentiable backups \citep{farquhar2017treeqn, guez2018learning}. Other model-based methods leverage learned abstract transition dynamics to perform lookahead rollouts \citep{oh2017value, silver2017predictron} or gradient-based trajectory optimization in latent space \citep{srinivas2018universal}. Thinker instead learns to plan by interacting with a learned world model before acting \citep{chung2023thinker}. In contrast, Looped Actor does not prescribe a particular planning algorithm or computation graph, rather it repeatedly applies a shared computational block and learns how this recurrent computation supports action selection.

Depth recurrence replaces such templates with repeated application of a shared module. Depth-recurrent policies provide a closely related approach. Deep Repeated ConvLSTM (DRC) agents repeatedly apply a stack of ConvLSTM modules several times per observation while carrying state across environment steps \citep{guez2019investigation}. Subsequent work has shown that these agents exhibit signatures of emergent planning \citep{bush2025interpreting} and improve with additional test-time iterations \citep{taufeeque2024planning}. In Sokoban, recurrence across environment steps also reduces the required within-step depth \citep{anokhin2026temporal}. \citep{ghugare2026role} prove that for some tasks, policies with more computation succeed and generalize to longer horizons where compute-limited policies fail. They support this empirically with a gated recurrent block iterated a fixed number of times per step, and identify state-dependent compute allocation as an open problem. The Looped Actor addresses this by allowing the number of recurrent iterations to vary with the current state.

Adaptive computation time \citep{graves2016adaptive,banino2021pondernet} has been studied extensively in recurrent and looped models \citep{dehghani2018universal,zhu2025scaling} but has seen limited use in RL. Continuous Thought Machines decouple internal ticks from the input and support certainty-based halting, with applications that include RL \citep{darlow2026continuous}, while Looped World Models apply an adaptive-depth looped transformer to transition prediction rather than action selection \citep{lu2026looped}. To our knowledge, the Looped Actor is the first RL policy to combine a looped-transformer backbone with state-dependent iteration counts, extending the attractor dynamics of looped reasoning models to actor learning.

\section{A Theoretical Motivation for Adaptive Computation}
\label{sec:adaptive_computation_theory}

We ask whether adapting computation to the state can reduce the expected computational cost of action selection while preserving optimal return. Building on the formal framework of \citet{ghugare2026role} and using classical tools from complexity theory \citep{sipser2006introduction}, we construct a family of MDPs for which a state-adaptive implementation can achieve the same optimal return with asymptotically lower expected computation than every optimal fixed-runtime implementation under a shared description-size bound. 

We represent a deterministic policy $\pi$ by a single-tape Turing machine that receives a binary encoding $s$ of an environment state and outputs a binary action. All implementations have descriptions shorter than $K_{\max}$ bits, where $K_{\max}$ is independent of input length.
Let $C_\pi(s)$ denote the number of machine steps used by the chosen implementation to select an action on input $s$. We say that an implementation has \emph{fixed runtime} at input length $m$ if,
\[
 C_\pi(s)=B_\pi(m)
 \qquad\text{for every }s\in\{0,1\}^{m},
\]
where $\{0,1\}^{m}$ is the set of all binary strings of length $m$, and $B_\pi(m)$ is the common number of computation steps spent on each such input. An adaptive implementation may use different amounts of computation on inputs of the same length. 
We consider a family of one-step MDPs $(\mathcal M_n)_{n\ge2}$, indexed by $n$. Let $\mu_n$ denote the initial-state distribution of $\mathcal M_n$, and let $J_n(\pi)$ denote the policy's expected return. Because each episode contains only one action, this return is its expected reward. The expected computational cost is defined by
\[
 \overline C_n(\pi)
 :=\mathbb{E}_{s\sim\mu_n}\!\left[C_\pi(s)\right].
\]

\begin{theorem}
\label{thm:adaptive_computation}
There exist constants $K_0\in\mathbb{N}$ and $c>0$, a single adaptive implementation $\pi_{\mathrm{ad}}$, and a family of finite one-step MDPs $(\mathcal M_n)_{n\ge2}$ with binary actions and rewards in $\{0,1\}$, such that the following holds for every fixed $K_{\max}\ge K_0$ and all sufficiently large $n$. The initial states comprise all $(n+1)$-bit strings, each with positive probability. The adaptive implementation has description length below $K_{\max}$ and satisfies
\begin{equation*}
 J_n(\pi_{\mathrm{ad}})=1,
 \qquad
 \overline C_n(\pi_{\mathrm{ad}})\le cn+1.
 \label{eq:adaptive_expected_cost}
\end{equation*}
Every optimal implementation $\pi$ within the same description-size bound that has fixed runtime at length $n+1$ instead satisfies $\overline C_n(\pi)>n^2$. An implementation within this bound that is optimal for every $n$ and has fixed runtime at every input length also exists. Let $\mathcal F_n$ denote the nonempty set of implementations
with description length below $K_{\max}$ that achieve
$J_n(\pi)=1$ and have fixed runtime at input length $n+1$.
Then,
\begin{equation*}
 \inf_{\pi\in\mathcal F_n}
 \frac{\overline C_n(\pi)}{\overline C_n(\pi_{\mathrm{ad}})}
 \ge \frac{n^2}{cn+1},
 \label{eq:adaptive_compute_separation}
\end{equation*}
and hence the computational advantage of adaptive computation grows without bound as \(n\to\infty\).
\end{theorem}

The construction captures the simple principle that some states are easy and occur frequently while others are rare but require substantially more computation. To achieve optimal return, a fixed-runtime policy must allocate enough computation to solve the hard states and incur that cost on every input regardless of difficulty. An adaptive policy can stop earlier on easy states, achieving the same optimal return with a computational advantage that grows without bound along the constructed MDP family. This motivates the state-dependent allocation of computation implemented by Looped Actor through adaptive halting. The full proof is provided in Appendix~\ref{app:adaptive_computation}.

\section{Looped Actor}
We instantiate our Looped Actor with an adaptation of a recently proposed looped transformer model for reasoning, called Fixed-Point Reasoning Model (FPRM) \citep{movahedi2026fixed}.
FPRM performs reasoning through iterative refinement of a latent representation. Given an embedded input \(\mathbf{x}\), a shared, input-conditioned Transformer block \(f_\theta\) repeatedly updates a latent state toward a fixed point, i.e., \(\mathbf{z}^{\star}=f_\theta(\mathbf{z}^{\star};\mathbf{x})\). Reusing the same parameters across iterations allows the model to increase its effective depth without increasing its parameter count. The model determines the number of loops itself for each input \(\mathbf{x}\) by a stability criterion, which we describe below.

To support stable computation over many loops, FPRM combines pre-normalization with residual scaling at both the sublayer and looping levels. Within a stack of \(K\) Transformer blocks, each attention or feed-forward sublayer applies

$$
\mathbf{h}^{k}
=
\alpha_1\mathbf{h}^{k-1}
+
\beta_1 g_{\theta^k}^{k}
\!\left(\operatorname{Norm}(\mathbf{h}^{k-1})\right),
\qquad k=1,\ldots,2K.
$$

Here, \(g_{\theta^k}^{k}\) denotes the transformation performed by sublayer \(k\), corresponding to either multi-head self-attention or a feed-forward network, with parameters \(\theta^k\) shared across recurrent iterations. The representation \(\mathbf{h}^{k-1}\) is the sublayer input. The input embedding is re-injected between iterations through the weighted combination \(\alpha_2\mathbf{h}^{2K}+\beta_2\mathbf{x}\). The learnable coefficients \(\alpha_1\) and \(\alpha_2\) control retention of the residual stream and recurrent state, respectively. The remaining coefficients are coupled as \(\beta_2=1-\alpha_2\alpha_1^{2K}\) and \(\beta_1=\beta_2(1-\alpha_1)/(1-\alpha_1^{2K})\). Under the stated boundedness assumptions and \(0\leq\alpha_1,\alpha_2<1\), this parameterization keeps the iterates bounded while retaining the signal-propagation benefits of pre-normalization. A depth-wise convolution at the beginning of each loop additionally supports local feature mixing for structured inputs. Together, these operations define the looped Transformer block \(f_\theta\).

While the original FPRM formulation defines halting via convergence of the relative residuals of the latent representations $\mathbf{z}^{\ell}$, we define halting via stabilization of the output distributions. More concretely, for a probabilistic policy predicting an output distribution $p_\ell = p_\ell(\cdot \mid s)$ conditioned on a state $s$ at loop $\ell$, e.g., Gaussian distributions with learned means or categorical distributions, we halt at the first loop satisfying
\begin{equation}
D_{\mathrm{KL}}(p_{\ell-1} | p_\ell) < \tau,
\label{eq:halting}
\end{equation}
where $D_{\mathrm{KL}}$ denotes the Kullback–Leibler (KL) divergence and $\tau$ is a predefined threshold. A similar approach has been taken in \citet{geiping2026scaling}, where the KL divergence between consecutive language token logit distributions is used as the halting mechanism. 

Finally, we train Looped Actor via backpropagation through time without truncation or deep supervision, contrary to the original FPRM model; we detail the supervised readouts in Section~\ref{sec:rl_pipeline}.

\subsection{Reinforcement Learning with Looped Actor}
\label{sec:rl_pipeline}

\paragraph{Problem setting.}
We formulate the RL problem as a goal-conditioned MDP described by the tuple $\{\mathcal{S}, \mathcal{A}, \mathcal{G}, \mathcal{T}, \mathcal{R}, \gamma\}$, where $\mathcal{S}$, $\mathcal{A}$, and $\mathcal{G}$ denote the state, action, and goal space, $\mathcal{T}$ the transition distribution, $\mathcal{R}$ the reward function, and $\gamma \in [0, 1)$ the discount factor. Let $s_t$ and $a_t$ denote the state and action at time $t$, where actions are sampled from the policy $\pi(a_t \mid s_t, g)$. Our objective is to learn a policy that maximizes the expected return $\mathbb{E}[\sum_t \gamma^t \mathcal{R}(s_t, a_t, g)]$, either through online interaction or from a fixed offline dataset $\mathcal{D}$. When the goal is implicit in the state, we omit $g$. Here, we use the output distribution at the halting loop as the policy, $\pi_\theta(a_t \mid s_t, g) = p_{\ell_\theta}(a_t)$, with $\ell_\theta(s_t, g)$ given by Eq.~\ref{eq:halting}. Halting introduces no additional parameters, and we train $\pi_\theta$ with standard RL objectives, treating halting as non-differentiable: gradients flow through the executed loops, and halted examples are frozen within a batch. %The same halting rule is used during acting and training, and $\ell_{\max}$ can be increased to scale test time compute.

\paragraph{Online RL.}
In the online setting, we train the Looped Actor with Proximal Policy Optimization (PPO)~\citep{schulman2017ppo}. The policy and value function share the looped block $f_\theta$ and are decoded from the latent at the halting loop, such that the critic is looped as well. Given rollouts collected with the previous policy $\pi_{\theta_\text{old}}$, we minimize
\begin{equation}
  \mathcal{L}(\theta) = \mathbb{E}_t\big[
  -\min\big(\rho_t \hat{A}_t, \mathrm{clip}(\rho_t, 1 - \epsilon, 1 + \epsilon)\hat{A}_t\big)
  + c_V \big(V_\theta(s_t) - \hat{V}_t\big)^2
  - c_H \mathcal{H}\big(\pi_\theta(\cdot \mid s_t)\big)\big],
  \label{eq:ppo}
\end{equation}
where $\rho_t = \pi_\theta(a_t \mid s_t) / \pi_{\theta_\text{old}}(a_t \mid s_t)$ denotes the importance ratio with clipping range $\epsilon$, using the action probability stored at collection time, $\hat{V}_t$ the $\lambda$-return~\citep{schulman2016gae}, $\hat{A}_t = r_t + \gamma \hat{V}_{t+1} - V_{\theta_\text{old}}(s_t)$ the batch-normalized advantage formulation from Brax~\citep{freeman2021brax}, $\mathcal{H}$ the policy entropy, and $c_V$ and $c_H$ weight the value and entropy terms. Note that $\pi_\theta$ is recomputed during each update, so its halting loop may differ from the one employed by the behavior policy.

\paragraph{Offline goal-conditioned RL.}
In the offline setting, we build on goal-conditioned implicit Q-learning (GCIQL)~\citep{kostrikov2022iql, park2025ogbench}, which samples value goals from the current state, future states of the same trajectory, or random states from $\mathcal{D}$, and actor goals from future states of the same trajectory, and assigns a reward of $-1$ at every step until the goal is reached. The value function $V_\xi$ and twin critics $Q_{\psi_1}, Q_{\psi_2}$ are MLPs trained with the expectile and temporal-difference losses of IQL, using a target critic for stability. We consider only the actor to be looped here, so that differences between models reflect only the actor architecture, and train with the DDPG+BC objective~\citep{fujimoto2021minimalist,park2024value}
\begin{equation}
  \mathcal{L}_\pi(\theta) = \mathbb{E}_{\mathcal{D}}\big[
  -\lambda_Q^{-1} \min\big(Q_{\psi_1}, Q_{\psi_2}\big)\big(s_t, \mu_\theta(s_t, g), g\big)
  - c_{\mathrm{BC}} \log \pi_\theta(a_t \mid s_t, g)\big],
  \label{eq:ddpgbc}
\end{equation}
where $\pi_\theta$ is a Gaussian with fixed unit standard deviation, following OGBench, so that the BC term equals $\tfrac12\lVert\mu_\theta(s_t,g)-a_t\rVert^2$ up to a constant, $\mu_\theta$ denotes its mean, clipped to the action bounds, $\lambda_Q$ normalizes the Q-values by their mean magnitude over the batch, and $c_{\mathrm{BC}}$ weights the behavior cloning term. The gradient of the Q-term propagates through all loops executed by the actor.

\section{Experiments}
The goal of the presented experiments is to isolate the benefits of looping and halting in RL policies and identify the mechanisms underlying these gains. We therefore compare Looped Actor with carefully chosen baselines, prioritizing controlled comparisons over state-of-the-art performance. Hyperparameters are provided in Appendix Section \ref{app:hps}.

\subsection{Environments}
We consider two environments in which successfully completing an episode requires multistep planning for most action decisions. We also include four environments from a recently proposed offline goal-conditioned benchmark. Together, these six environments span online and offline RL and include both continuous and discrete action spaces for a total of $22$ downstream tasks. Exemplary visualizations for all six environments can be found in Appendix \ref{app:env_images}.
\paragraph{Boxoban} 
The first environment is based on Sokoban \citep{sokoban,sokoban2}. Each puzzle takes place on a two-dimensional board, where an agent must push boxes to goal locations. The difficulty arises from obstacles and walls, combined with the restriction that boxes can be pushed but not pulled. As a result, the agent can enter dead-end states from which recovery is impossible. Solving a Sokoban puzzle therefore requires planning multiple steps ahead to avoid such states. In this paper, we use the Boxoban version introduced by \citet{guez2019investigation}, which fixes the board size at $10\times10$. We consider the unfiltered set, comprising 900k puzzles for training, 100k for validation, and 1k for testing. 

\paragraph{Rush hour}
We consider the puzzle game of Rush Hour which is played on a $6\times6$ board occupied by cars and trucks of length $2$ and $3$, respectively. The objective is to move the vehicles to clear a path for a designated car to reach the exit. This represents a very challenging planning game. In fact, its generalization to arbitrarily large boards is proven to be PSPACE-complete \citep{flake2002rush}. In this paper, we develop an RL environment based on the rush repository\footnote{\url{https://github.com/fogleman/rush}}. Our training, validation, and test sets contain 2M, 10K, and 10K puzzles, respectively. Each game has an optimal solution length of at most $15$ moves, defined as the minimum number of moves required to solve it. 

\paragraph{OGBench}
We extend our study to offline goal-conditioned RL with continuous actions using four tabletop manipulation environments from OGBench~\citep{park2025ogbench}. The Cube double and Cube triple environments require rearranging cubes into a target configuration. The Puzzle 3$\times$3 environment requires pressing buttons on a grid to reach a target color pattern, where each press toggles the pressed button and its neighbors, as in the game Lights Out~\citep{anderson1998turning}. The Scene environment requires bringing a cube, a drawer, a window, and two buttons into a target configuration, where the buttons (un-)lock the drawer and window. We train on the state-based play datasets, collected by scripted policies performing randomly chosen subtasks, and evaluate on the five goal classes provided per environment, for a total of 20 tasks. Solving these tasks requires both high-level planning of subtasks and fine-grained continuous control of the robot arm.

\subsection{Model and Baseline Specifications}
We limit the number of loops for our Looped Actor to a maximum of $16$ iterations during training and compare it with two baselines. The \textbf{Iso-Parameters} baseline uses the same transformer block as Looped Actor but applies it only once, preserving the parameter count. The \textbf{Iso-FLOPs} baseline stacks $16$ copies of this block without sharing parameters across blocks. It therefore matches the training compute of Looped Actor at the maximum loop count, but has $16$ times as many parameters as either Looped Actor or the Iso-Parameters baseline. We choose the hidden dimension of the blocks to be $128$, resulting in a total of $0.5$M parameters for Looped Actor and Iso-Parameters, and around $8$M parameters for Iso-FLOPs. Finally, we include \textbf{DRC} as a baseline on Boxoban and Rush Hour (not on OGBench as it requires a grid-structured input to its convolutional layers) because it is the closest to our approach. Since Looped Actor aims to perform multistep planning from the current state without memory, we further adapt DRC to the same stateless setting. We use the depth-recurrent variant DRC$(3,3)$, which performed best across all environments considered in \citet{guez2019investigation}. With approximately 1M parameters, DRC$(3,3)$ is roughly twice the size of both Looped Actor and the Iso-Parameters baseline.

For our transformer-based models, we adopt the tokenization scheme proposed by \citet{dong2026tql}. Specifically, each state dimension is treated as a separate token and mapped to a latent embedding through a shared linear encoder. For goal-conditioned OGBench experiments, we append the goal tokens to the state-token sequence. For the Puzzle 3$\times$3 environment, we instead concatenate corresponding state and goal dimensions and apply the linear encoder to each resulting pair.

\subsection{Results}
We report the success rate as a function of environment steps for Boxoban in \fref{fig:sokoban_learning_curves} and Rush Hour in \fref{fig:rushhour_success_vs_env_steps}, showing the mean and standard deviation across three random seeds. Looped Actor substantially outperforms the Iso-FLOPs and DRC baselines despite having $2$ to $16$ times less parameters. The Iso-Parameters baseline shows no meaningful learning on Boxoban and, although it achieves a nonzero success rate on Rush Hour, performs substantially worse than both Looped Actor and Iso-FLOPs. Together, these results demonstrate that looped reasoning enables strong performance in environments requiring complex, multistep planning.

\begin{figure}[ht]
    \centering

    \begin{minipage}[t]{0.47\textwidth}
        \centering
        \begin{minipage}[c][5cm][c]{\linewidth}
            \centering
            \includegraphics[
                width=\linewidth,
                height=5cm,
                keepaspectratio
            ]{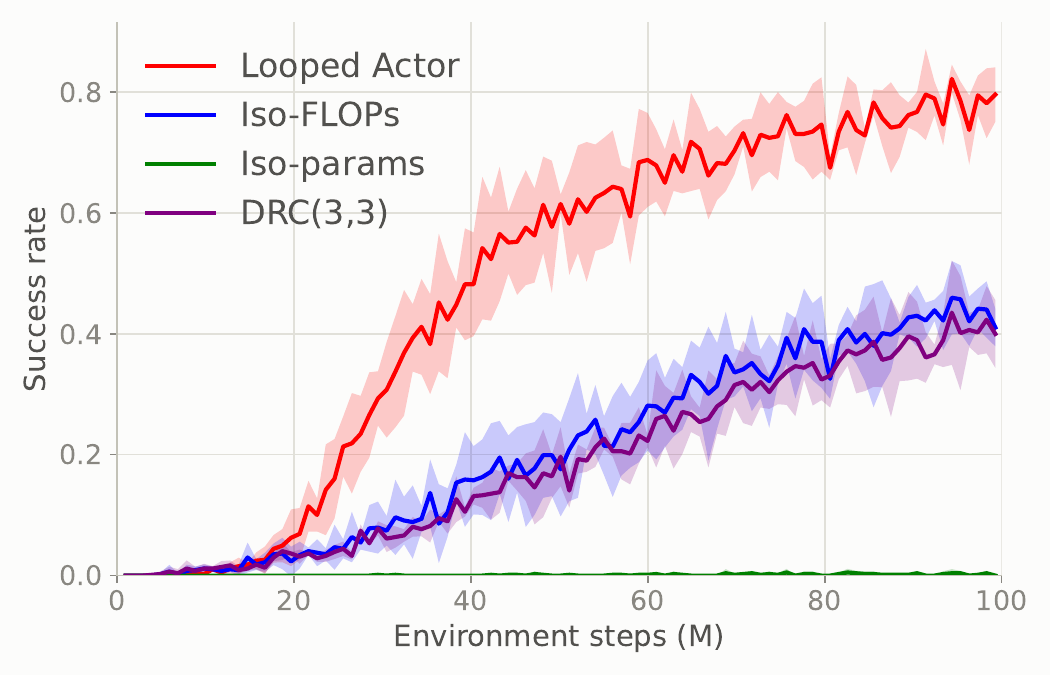}
        \end{minipage}
        \caption{Boxoban success rate on unfiltered games during training for Looped Actor, Iso-FLOPs, Iso-Parameters, and DRC$(3,3)$ models.}
        \label{fig:sokoban_learning_curves}
    \end{minipage}
    \hfill
    \begin{minipage}[t]{0.47\textwidth}
        \centering
        \begin{minipage}[c][5cm][c]{\linewidth}
            \centering
            \includegraphics[
                width=\linewidth,
                height=5cm,
                keepaspectratio
            ]{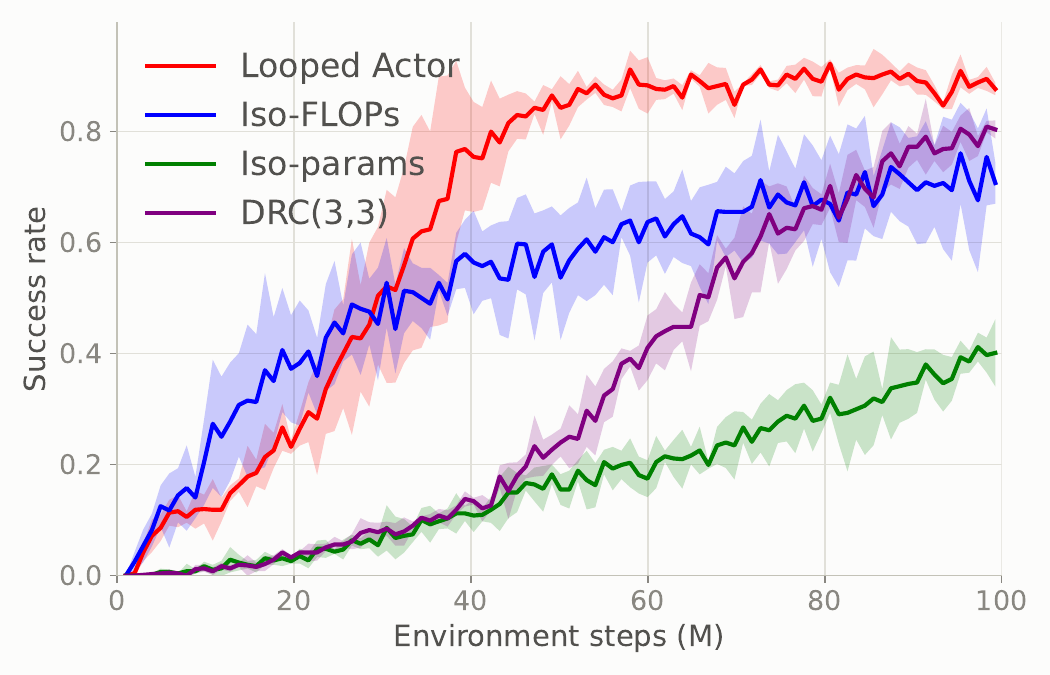}
        \end{minipage}
        \caption{Rush Hour success rate during training for Looped Actor, Iso-FLOPs, Iso-Parameters, and DRC$(3,3)$ models.}
        \label{fig:rushhour_success_vs_env_steps}
    \end{minipage}

\end{figure}

To assess whether Looped Actor performs well in environments that require less multistep planning than Boxoban or Rush Hour, we evaluate all Looped Actor, Iso-Parameters, and Iso-FLOPs on the four OGBench environments considered in this paper. \tref{tab:main} reports the mean and standard deviation across five random seeds. Looped Actor ties or slightly outperforms Iso-FLOPs in all four environments. Iso-Parameters performs substantially worse, achieving zero success on at least one downstream task in three of the four environments. We provide the success rates during the full training runs in \fref{fig:ogbench_learning_curves_nods_means} of the Appendix. 

Together, these results show that Looped Actor matches the performance of Iso-FLOPs with substantially fewer parameters and outperforms it by a large margin in environments where most action decisions require multistep planning. Iso-Parameters consistently underperforms both models across all environments.

\begin{table}[h]
\centering
\caption{Success rate of the Looped Actor and the Iso-parameters and Iso-FLOPs baselines on OGBench tasks (\%, mean $\pm$ std over 5 seeds, final checkpoint at 1M gradient steps).}
\label{tab:main}
\begin{tabular}{lccc}
\toprule
Environment / task & Iso-params & Iso-FLOPs ($16\times \#$params) & Looped Actor \\
\midrule
\rowcolor{gray!15} \textbf{Cube double} & 15.2 $\pm$ 7.6 & 66.9 $\pm$ 4.2 & 67.7 $\pm$ 4.6 \\
\quad single pnp & 50.8 $\pm$ 22.0 & 95.6 $\pm$ 3.0 & 97.2 $\pm$ 1.8 \\
\quad double pnp 1 & 7.2 $\pm$ 7.6 & 88.8 $\pm$ 4.6 & 89.6 $\pm$ 13.1 \\
\quad double pnp 2 & 3.2 $\pm$ 4.6 & 84.8 $\pm$ 8.3 & 93.2 $\pm$ 5.2 \\
\quad swap & 0.0 $\pm$ 0.0 & 12.8 $\pm$ 3.0 & 6.4 $\pm$ 2.6 \\
\quad stack & 14.8 $\pm$ 14.1 & 52.4 $\pm$ 8.6 & 52.0 $\pm$ 7.7 \\
\midrule
\rowcolor{gray!15} \textbf{Cube triple} & 2.7 $\pm$ 5.2 & 47.4 $\pm$ 7.9 & 50.1 $\pm$ 6.1 \\
\quad single pnp & 8.8 $\pm$ 15.3 & 93.2 $\pm$ 10.0 & 96.8 $\pm$ 1.8 \\
\quad triple pnp & 2.8 $\pm$ 6.3 & 79.6 $\pm$ 11.1 & 80.4 $\pm$ 20.0 \\
\quad pnp from stack & 2.0 $\pm$ 4.5 & 43.6 $\pm$ 18.8 & 52.4 $\pm$ 18.0 \\
\quad cycle & 0.0 $\pm$ 0.0 & 4.0 $\pm$ 2.4 & 8.0 $\pm$ 2.0 \\
\quad stack & 0.0 $\pm$ 0.0 & 16.8 $\pm$ 10.4 & 12.8 $\pm$ 7.2 \\
\midrule
\rowcolor{gray!15} \textbf{Puzzle 3$\times$3} & 89.9 $\pm$ 7.5 & 91.4 $\pm$ 12.5 & 91.4 $\pm$ 7.4 \\
\quad task 1 & 99.6 $\pm$ 0.9 & 94.4 $\pm$ 10.4 & 96.8 $\pm$ 5.2 \\
\quad task 2 & 94.0 $\pm$ 5.1 & 90.8 $\pm$ 14.1 & 94.0 $\pm$ 6.9 \\
\quad task 3 & 87.2 $\pm$ 7.8 & 91.2 $\pm$ 12.4 & 86.0 $\pm$ 11.7 \\
\quad task 4 & 83.2 $\pm$ 14.2 & 88.4 $\pm$ 17.1 & 88.0 $\pm$ 6.8 \\
\quad task 5 & 85.6 $\pm$ 16.0 & 92.4 $\pm$ 10.4 & 92.0 $\pm$ 9.6 \\
\midrule
\rowcolor{gray!15} \textbf{Scene play} & 24.2 $\pm$ 4.9 & 68.2 $\pm$ 4.6 & 69.6 $\pm$ 5.5 \\
\quad open & 62.4 $\pm$ 18.7 & 99.2 $\pm$ 1.1 & 99.6 $\pm$ 0.9 \\
\quad unlock and lock & 4.8 $\pm$ 2.7 & 96.0 $\pm$ 4.0 & 98.0 $\pm$ 2.4 \\
\quad rearrange medium & 50.0 $\pm$ 7.9 & 90.4 $\pm$ 5.2 & 92.8 $\pm$ 6.6 \\
\quad put in drawer & 3.6 $\pm$ 1.7 & 35.2 $\pm$ 5.0 & 43.6 $\pm$ 21.5 \\
\quad rearrange hard & 0.0 $\pm$ 0.0 & 20.0 $\pm$ 14.8 & 14.0 $\pm$ 13.1 \\
\midrule
\rowcolor{gray!15} \textbf{Overall average} & 33.0 $\pm$ 3.4& 68.5 $\pm$ 4.7 & 69.7 $\pm$ 2.4 \\
\bottomrule
\end{tabular}
\vspace{-1em}
\end{table}

\paragraph{Multistep planning.}
\begin{wrapfigure}{r}{0.45\textwidth}
    \centering
    \includegraphics[width=\linewidth]{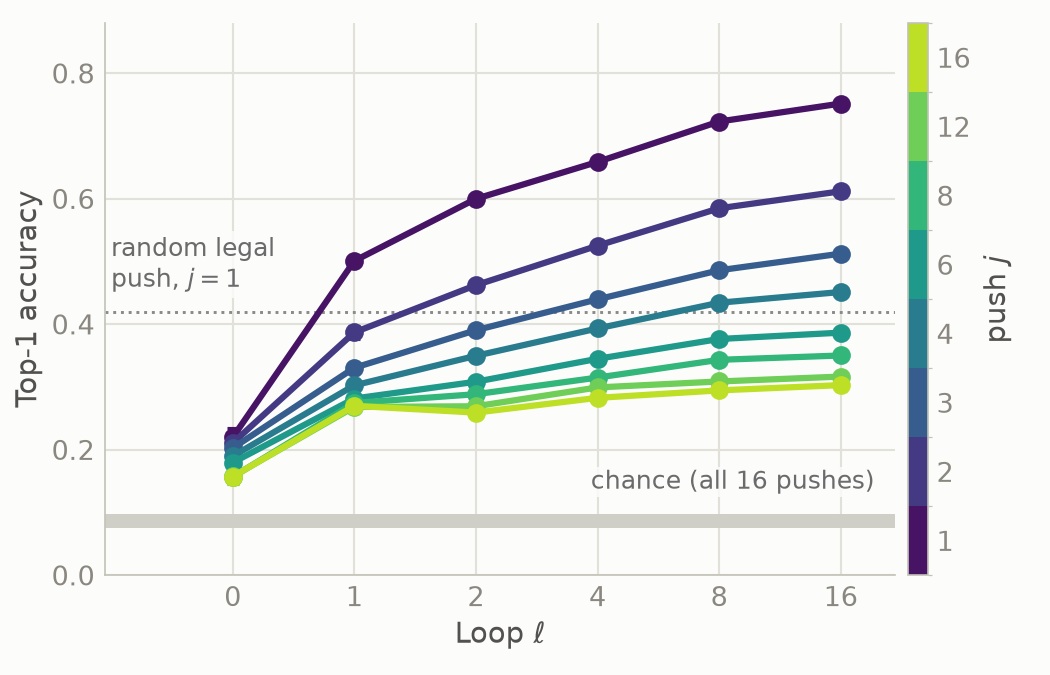}
    \caption{Accuracy in decoding the optimal plan from Looped Actor latents on Boxoban.}
    \label{fig:boxoban_plan_depth_lines}
\end{wrapfigure}
To test whether the model builds a multistep plan over looping, we probe the latent states for knowledge of the optimal plan. We take 50k held-out Boxoban levels (25k unfiltered validation, 25k medium validation) and find the move-optimal solutions using a shortest-path solver, recording the sequences of box pushes. For each future box move in $j \in \{1,2,3,4,6,8,12,16\}$ and loop $\ell \in \{1,2,4,8,16\}$, we train a small attentive probe that reads all $101$ tokens of the latent state $\mathbf{z}^{\ell}$ and predicts the $j$-th move as a (box, direction) pair among the $16$ candidates. The probe is counted correct if its top choice is the $j$-th move of an optimal solution. As a control we train a probe on the input embedding of the puzzle $\mathbf{x}$, i.e. without any looping. Levels are split $70/10/20$ into train, validation, and test, and we report mean and standard deviation test accuracy across the three Looped Actor checkpoints on Boxoban, each checkpoint itself averaged over three seeded probes. 

\fref{fig:boxoban_plan_depth_lines} shows the optimal plan decodability as latent looping increases. Plans are present from the first loop and are refined by further loops: at every move $j$ the decoded latent at $\ell=1$ already beats the no-looping control ($\ell=0$), and accuracy rises almost monotonically in $\ell$. Additionally, the benefit of looping is most prominent on nearby decisions. From $\ell=1$ to $\ell=16$, accuracy on the next push rises from $0.50$ to $0.75$, on the fourth push from $0.30$ to $0.45$, and on the sixteenth push only from $0.27$ to $0.30$. 
Accuracy on the sixteenth push is still well above both random chance ($0.0625$) and the input-embedding probe ($0.16$). 
We also compare to a random guess among the pushes that are non-colliding, which scores $0.42$ on the immediate push $j=1$. 
The input-embedding probe ($0.27$) falls short of this, the latent after one loop ($0.50$) performs better, and accuracy increases from there. Together, these results suggest that Looped Actor spends its iterative computation sharpening local decisions while building a coarse representation of a long-horizon plan.

\begin{figure}[ht]
    \centering
    \begin{minipage}[t]{0.47\textwidth}
        \centering
        \begin{minipage}[c][5cm][c]{\linewidth}
            \centering
            \includegraphics[
                width=\linewidth,
                height=5cm,
                keepaspectratio
            ]{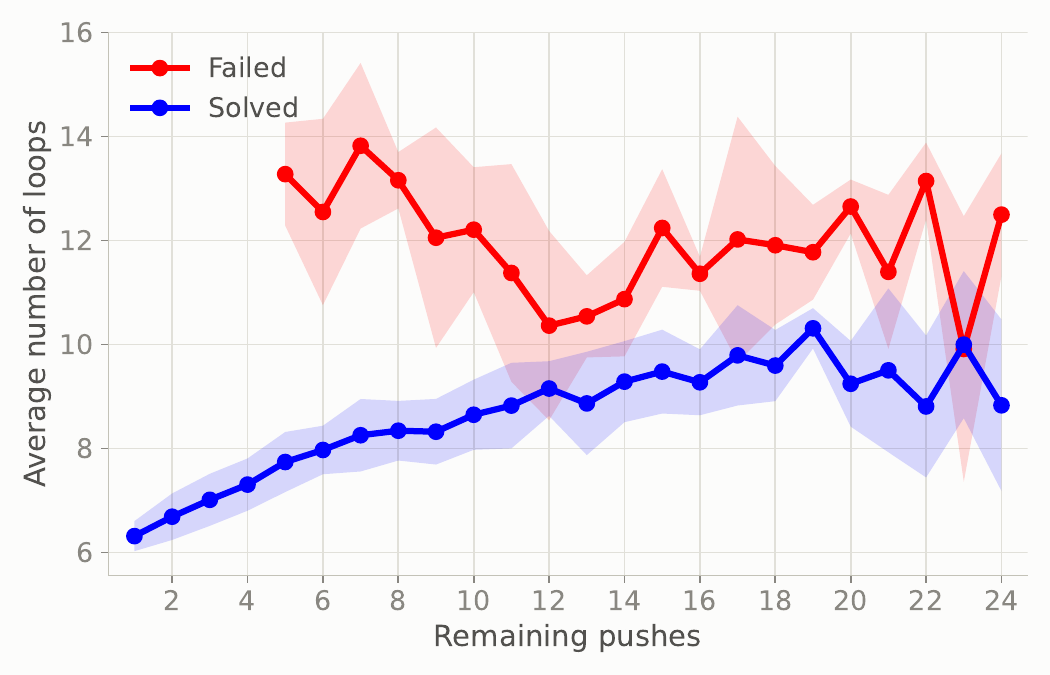}
        \end{minipage}
        \caption{Loop count for individual states plotted against the corresponding number of box pushes, for both failed and solved episodes.}
    \label{fig:boxoban_loops_vs_pushes}
    \end{minipage}
    \hfill
    \begin{minipage}[t]{0.47\textwidth}
        \centering
        \begin{minipage}[c][5cm][c]{\linewidth}
            \centering
            \includegraphics[
                width=\linewidth,
                height=5cm,
                keepaspectratio
            ]{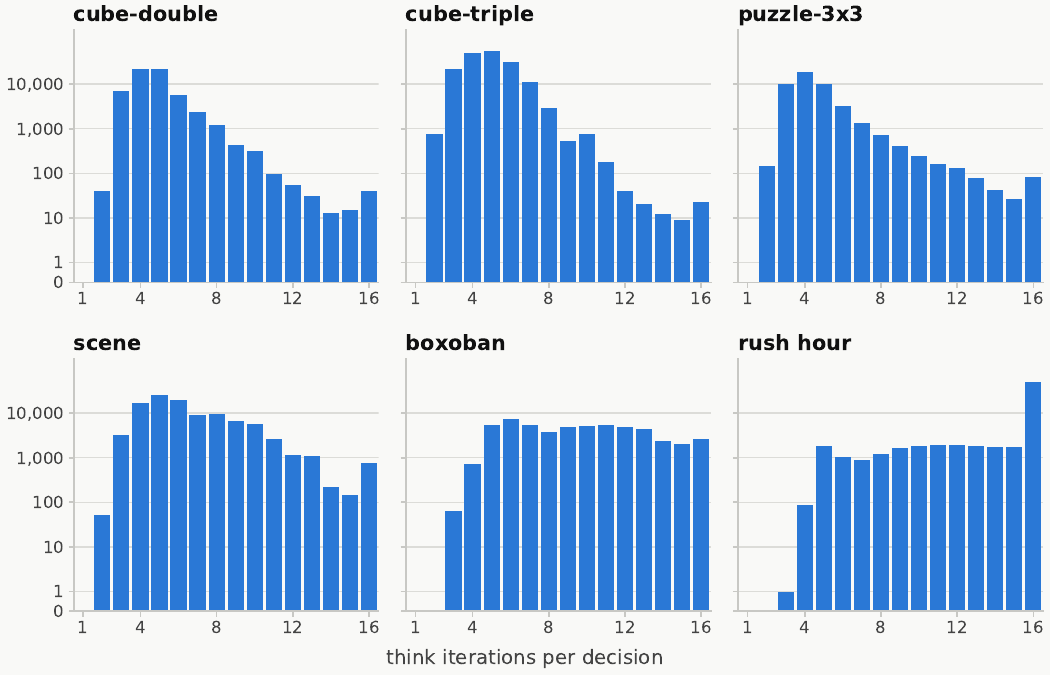}
        \end{minipage}
        \caption{Histograms of loop iterations of Looped Actor for all $6$ environments considered in this paper.}
        \label{fig:loop_hist_grid}
        
    \end{minipage}
\vspace{-1em}
\end{figure}

\paragraph{Learned compute allocation.}
Training the policy to stabilize its output distribution for each action associated with a state raises the natural question of whether the amount of compute spent to reach this stable point, i.e., the number of loops, correlates with the difficulty of the state in a meaningful way. To test this, we consider the three final Looped Actor checkpoints for Boxoban. We define the difficulty of a particular Boxoban state as the number of remaining box pushes needed to solve the game in minimal time. The test set contains states requiring between one and 24 remaining pushes.

\fref{fig:boxoban_loops_vs_pushes} plots the average number of loops of trained Looped Actor as a function of the number of remaining pushes, for both successful and failed episodes. We can see that when the agent is able to solve an episode, the average loop count increases almost perfectly monotonically between $1$ and $19$ remaining pushes, ranging from $6$ loops to more than $10$ loops. Notably, if the agent is unable to solve an episode, the loop count does not correlate meaningfully with the difficulty of a state. 

\begin{wrapfigure}{r}{0.45\textwidth}
    \centering
    \includegraphics[width=\linewidth]{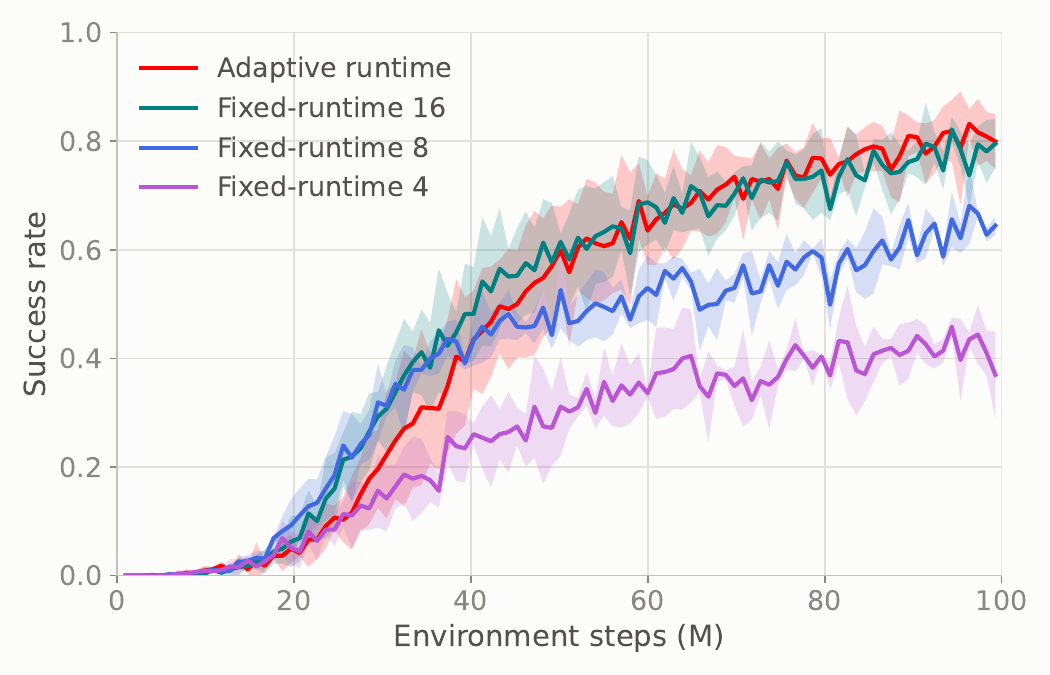}
    \caption{Adaptive vs Fixed-runtime success rates for Looped Actor on Boxoban.}
    \label{fig:boxoban_cap_ablation}
    \vspace{-1em}
\end{wrapfigure}
We further examine the distribution of per-state loop counts in the learned Looped Actor. \fref{fig:loop_hist_grid} shows these distributions across Boxoban and the five OGBench environments we consider. We can see that some states require substantially more loops. This is particularly pronounced in Boxoban and scene, which require the most planning and exhibits a long tail toward higher loop counts. Moreover, while Looped Actor is trained with a maximum loop count of $16$, \fref{fig:loop_hist_grid} shows that the effective average depth is much smaller. Thus, Looped Actor uses both fewer parameters and fewer inference FLOPs than Iso-FLOPs. 

To assess whether adaptive computation compromises performance, we compare our state-adaptive Looped Actor with fixed-runtime variants trained and evaluated using $4$, $8$, or $16$ loops. \fref{fig:boxoban_cap_ablation} shows success rates throughout training. Using fewer than $16$ loops substantially reduces performance, whereas the fixed-runtime variant with $16$ loops matches our state-adaptive Looped Actor. Yet more than $95\%$ of trajectories in the adaptive model halt before reaching $16$ loops, yielding a speedup of over $60\%$ relative to the fixed-runtime counterpart. These results demonstrate the computational benefits of adaptive computation: the policy devotes additional computation to states that require it, avoiding the cost of running all $16$ loops at every state.

\section{Conclusion}
We introduced Looped Actor, a depth-recurrent reasoning model policy that allocates computation adaptively across states by repeatedly applying a shared transformer block until its action distribution stabilizes. We provided a theoretical motivation for state-dependent computation, constructing a family of MDPs in which an adaptive policy achieves optimal return with asymptotically lower expected computation than any optimal fixed-runtime policy under the same description-size constraint. Empirically, Looped Actor matches or exceeds a larger untied transformer (with 16 times as many parameters) across 22 tasks spanning online and offline RL. While the largest improvements are on problems requiring multistep planning, we also show that Looped Actor retains competitive performance on high dimensional continuous control tasks. Through latent probes, we show that successive loops progressively refine representations of future optimal actions. 
In addition, we showed that the number of loops increases with planning difficulty in successful multistep planning episodes. 
Finally, we showed empirically that loop counts vary widely across states, highlighting the potential computational savings of adapting computation to each state rather than applying a large, fixed budget uniformly.
Together, these results suggest that Looped Actor offers a simple yet effective approach to equip
RL agents with adaptive computation and improve their planning capabilities.

Several limitations and open questions remain. Although our theoretical construction in Section~\ref{sec:adaptive_computation_theory} shows that compute-adaptive policies can be substantially more efficient than fixed-time policies, there is no guarantee that Looped Actor learns a compute-optimal policy in practice. Developing approaches that further improve the computation–reward Pareto frontier remains an important direction for future work. Our formulation also focuses on within-step computation without carrying internal reasoning across environment steps, leaving open how adaptive depth should interact with memory and partial observability.

\subsubsection*{Acknowledgments}
This work was supported by the Hector Foundation and the Department of the Air Force Artificial Intelligence Accelerator and was accomplished under Cooperative Agreement Number FA8750-19-2-1000. The views and conclusions contained in this document are those of the authors and should not be interpreted as representing the official policies, either expressed or implied, of the Department of the Air Force or the U.S. Government. The U.S. Government is authorized to reproduce and distribute reprints for Government purposes notwithstanding any copyright notation herein.

\bibliography{iclr2027_conference}
\bibliographystyle{iclr2027_conference}

\clearpage
\appendix
\begin{center}
{\bf Supplementary Material for:}\\
Looped Actor: Depth-Recurrent Reasoning Models for Reinforcement Learning
\end{center}

\section{Experimental Details}
All experiments were run on NVIDIA H100 and B200 GPUs. Both the environments and the training code were implemented in JAX \citep{jax2018github}.

\subsection{Visualization of Environments}
\label{app:env_images}

\begin{figure}[ht]
    \centering
    \begin{subfigure}[t]{0.3\textwidth}
        \centering
        \includegraphics[width=\linewidth]{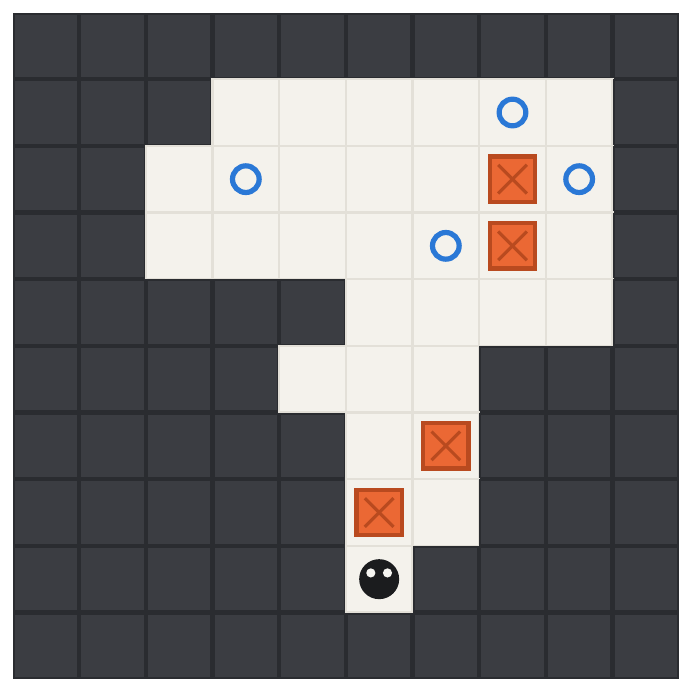}
        \caption{Example of a Boxoban game.}
        \label{fig:boxoban_examples}
    \end{subfigure}
    \hspace{1em}
    \begin{subfigure}[t]{0.32\textwidth}
        \centering
        \includegraphics[width=\linewidth]{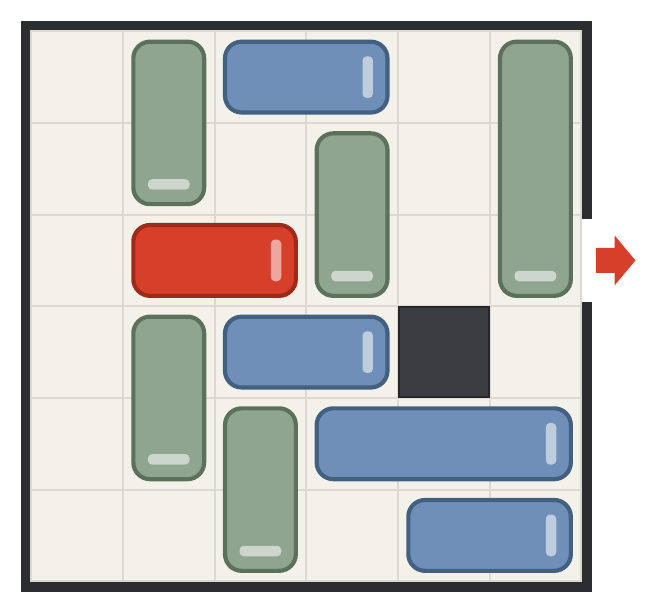}
        \caption{Example of a Rush Hour game.}
        \label{fig:rushhour_easy_test}
    \end{subfigure}
    \caption{Examples of the two game environments.}
    \label{fig:game_examples}
\end{figure}

\begin{figure}[ht]
\centering
\includegraphics[width=\textwidth]{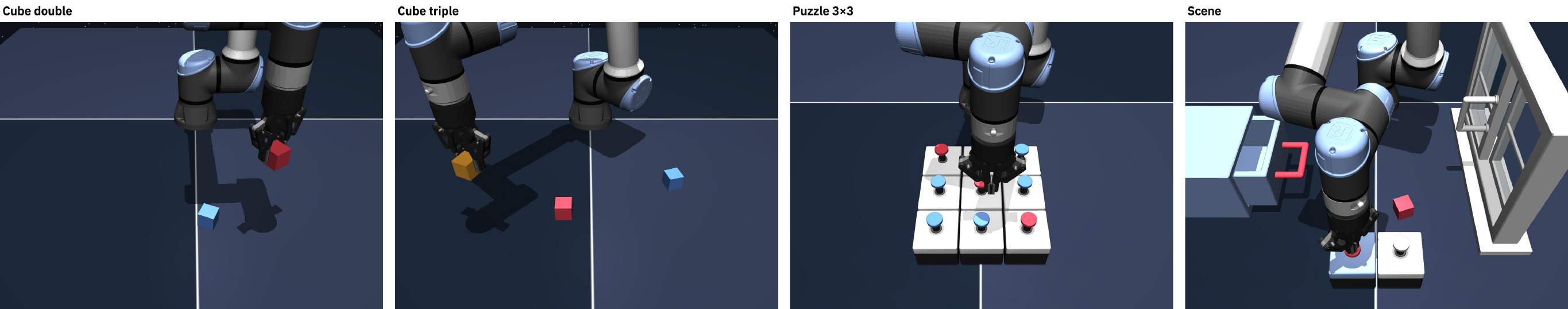}
\caption{OGBench environments for Cube double, Cube triple, Puzzle 3x3, and scene.}
\label{fig:ogbench_examples}
\end{figure}

\subsection{Hyperparameters}
\label{app:hps}
We chose the same fixed hyperparameters for all Transformer models considered in this paper. \tref{tab:hparams-model} provides the model hyperparameters, while \tref{tab:hparams-train} shows the training hyperparameters.

\begin{table}[ht]
\centering
\small
\begin{tabular}{@{}lccc@{}}
\toprule
Hyperparameter & OGBench & Boxoban & Rush Hour \\
\midrule
Layers per iteration & 2 & 2 & 2 \\
Model width & 128 & 128 & 128 \\
Attention heads & 4 & 4 & 4 \\
MLP expansion ratio & 4 & 4 & 4 \\
Grid convolution & No & Yes & Yes \\
Train iterations (min/max) & 2/16 & 2/16 & 2/16 \\
Halting threshold $\tau$ & $10^{-3}$ & $10^{-3}$ & $10^{-3}$ \\
\bottomrule
\end{tabular}
\caption{Model hyperparameters.}
\label{tab:hparams-model}
\end{table}

\begin{table}[ht]
\centering
\small
\begin{tabular}{@{}lcc@{}}
\toprule
Hyperparameter & OGBench & Boxoban / Rush Hour \\
\midrule
Algorithm & GCIQL & PPO \\
Optimizer & Adam & Adam \\
Learning rate & $3\times10^{-4}$ & $10^{-4}$ \\
Gradient clipping (global norm) & 1.0 & 1.0 \\
Batch / minibatch size & 1024 & 1024 \\
Training budget & $10^6$ gradient steps & $10^8$ environment steps \\
Discount $\gamma$ & 0.99 & 0.99 \\
\midrule
Parallel environments & -- & 1024 \\
Rollout steps per environment & -- & 64 \\
Epochs per rollout & -- & 4 \\
GAE $\lambda$ & -- & 0.95 \\
PPO clip $\epsilon$ & -- & 0.3 \\
Entropy coefficient & -- & 0.01 \\
Value loss coefficient & -- & 0.25 \\
IQL expectile & 0.9 & -- \\
Target EMA $\tau$ & 0.005 & -- \\
DDPG+BC $\alpha$ & 1.0 & -- \\
\bottomrule
\end{tabular}
\caption{Optimization hyperparameters. Boxoban and Rush Hour
share all settings listed here.}
\label{tab:hparams-train}
\end{table}

\section{Proofs}
\subsection{Proof of Theorem~\ref{thm:adaptive_computation}}
\label{app:adaptive_computation}

\paragraph{Proof idea.}
We construct a decision problem with frequent easy states and rare states on which the policy must solve a computationally demanding problem. The key lower bound comes from one state for each candidate policy: a state containing that policy's own program description. An optimal policy cannot finish within $n^2$ steps on this state. If its runtime is fixed, it must therefore spend more than $n^2$ steps on every state of the same length. An adaptive policy can instead stop quickly on the frequent easy states.

\paragraph{Computational conventions.}
Policy implementations are deterministic single-tape Turing machines that halt on every finite binary input. Fix an explicit, polynomial-time parsable binary encoding of finite transition tables, with state and tape-symbol identifiers written in binary. Let $\langle M\rangle$ denote the description of machine $M$, and let $M_\pi$ denote the machine implementing policy $\pi$. The size constraint $|\langle M_\pi\rangle|<K_{\max}$ counts the program description, excluding its input and working tape contents. Runtime counts all executed machine transitions.

\begin{proof}
\textbf{1. Define the correct action on the hard branch.}
We adapt the diagonalization in the proof of Theorem~A.1 of \citet{ghugare2026role}. Let $U$ be a universal simulator: given a program description and an input, it executes that program on that input.

For $n\ge2$ and $x\in\{0,1\}^n$, check whether $x$ has the form
\[
 x=1^\ell0p0^k,
 \qquad |p|=\ell,\qquad k=n-2\ell-1\ge0,
\]
where $p$ is a valid machine description. Here juxtaposition means concatenation, $1^\ell$ means $\ell$ consecutive ones, and $0^k$ means $k$ consecutive zeros. The prefix specifies the length of the program description, and the trailing zeros pad the string to length $n$.

For a valid encoding, use $U$ to execute the program described by $p$ on input $1x$, stopping after at most $n^2$ steps of the encoded program. Here $p$ specifies the instructions, while $1x$ is the complete state on which those instructions are executed. Define
\[
 h_n(x)=
 \begin{cases}
 1-a,&\text{if the program outputs }a\in\{0,1\}\text{ within }n^2\text{ steps},\\
 0,&\text{otherwise, including invalid encodings.}
 \end{cases}
\]
Thus $h_n$ always returns a binary answer. It deliberately disagrees with the encoded program whenever that program finishes within the cutoff. The cutoff counts steps of the encoded program; evaluating $h_n$ can take more actual steps because simulation has overhead.

Using padded configurations and a prescribed sequence of full tape sweeps, parsing and simulation can follow a schedule determined solely by $n$. All $n^2$ simulation rounds are executed, retaining the recorded result after simulated halting. This yields a polynomial-time implementation whose actual runtime is identical on all length-$n$ inputs.

\textbf{2. Construct two implementations of the same policy.}
Write the observed state as $s=bx$, where $b\in\{0,1\}$ and $x\in\{0,1\}^n$. The adaptive implementation $\pi_{\mathrm{ad}}$ scans the input and uses $b$ to select a branch: if $b=0$, it outputs $0$ immediately after the scan; if $b=1$, it computes and outputs $h_n(x)$. Consequently, for some constant $c>0$ and an integer-valued polynomial $q(n)\ge n^2$,
\begin{equation}
 C_{\pi_{\mathrm{ad}}}(0x)\le cn,
 \qquad
 C_{\pi_{\mathrm{ad}}}(1x)\le q(n)
 \quad\text{for every }x\in\{0,1\}^n,\ n\ge2.
 \label{eq:adaptive_branch_costs}
\end{equation}
The easy branch only scans the input while remembering $b$. The polynomial $q$ bounds all work on the hard branch, including parsing and simulation overhead; a sufficiently large integer multiple of $(n+1)^d$ suffices for a sufficiently large fixed integer $d$.

The fixed-runtime implementation $\pi_{\mathrm{fix}}$ computes $h_n(x)$ on every input, regardless of $b$, using the fixed schedule above. It then outputs $0$ if $b=0$, discarding the computed answer, and outputs $h_n(x)$ if $b=1$. Make this final selection take the same number of steps in either case. This program selects the same actions as $\pi_{\mathrm{ad}}$, but its runtime depends only on input length. For inputs shorter than three bits, both machines perform a fixed scan and output $0$.

Both programs have constant descriptions independent of $n$. Choose
\[
 K_0=1+\max\bigl\{
 |\langle M_{\pi_{\mathrm{ad}}}\rangle|,
 |\langle M_{\pi_{\mathrm{fix}}}\rangle|
 \bigr\}.
\]
Then both descriptions are shorter than every $K_{\max}\ge K_0$. This particular fixed-runtime implementation establishes existence; the lower bound below will apply to every optimal fixed-runtime implementation, regardless of its algorithm.

\textbf{3. Define the rewards and make the hard branch rare.}
The initial state space is $\mathcal S_n=\{0,1\}^{n+1}$. At state $bx$, action $0$ is rewarded when $b=0$, while action $h_n(x)$ is rewarded when $b=1$:
\[
 r_n(bx,a)=
 \begin{cases}
 \mathbf{1}\{a=0\},&b=0,\\
 \mathbf{1}\{a=h_n(x)\},&b=1.
 \end{cases}
\]
Here $\mathbf{1}\{E\}$ is one if $E$ holds and zero otherwise. After the action, the MDP transitions deterministically to a separate terminal state $\bot$, with no further action or reward. Thus $h_n(x)$ specifies the correct action, not the reward itself. Both constructed implementations always choose correctly and achieve optimal return $1$.

Sample $x$ uniformly from $\{0,1\}^n$ and independently choose $b=1$ with probability $1/q(n)$. Every initial state has positive probability:
\[
 \mu_n(1x)=\frac{2^{-n}}{q(n)},
 \qquad
 \mu_n(0x)=2^{-n}\left(1-\frac1{q(n)}\right).
\]
Combining these probabilities with ~\eqref{eq:adaptive_branch_costs} gives
\[
 \overline C_n(\pi_{\mathrm{ad}})
 \le \left(1-\frac1{q(n)}\right)cn
      +\frac1{q(n)}q(n)
 \le cn+1.
\]
The hard branch can be expensive, but its contribution to expected computation is at most one because it occurs with probability $1/q(n)$. The two programs, $q$, and the MDP family are all fixed independently of $K_{\max}$.

\textbf{4. Use one state to lower-bound the common runtime.}
Fix $K_{\max}\ge K_0$ and $n\ge2K_{\max}+1$. Consider any optimal fixed-runtime implementation $\pi\in\mathcal F_n$. It may use any algorithm; it need not evaluate $h_n$ by simulation. Let $p=\langle M_\pi\rangle$ be its own program description and $\ell=|p|<K_{\max}$. The choice of $n$ ensures that this description fits inside the string
\[
 x_\pi=1^\ell0p0^{\,n-2\ell-1}.
\]
The corresponding state $1x_\pi$ belongs to $\mathcal S_n$ and has positive probability.

Suppose, for a contradiction, that the common runtime satisfies $B_\pi(n+1)\le n^2$. On state $1x_\pi$, the policy then finishes within $n^2$ steps and outputs some action $a=\pi(1x_\pi)$. To determine the rewarding action at this state, the rule $h_n$ extracts the encoded program and runs it on $1x_\pi$. The encoded program is exactly $M_\pi$, and its input is exactly the state on which we are evaluating $\pi$. Therefore the simulation finishes within the cutoff with the same output $a$, giving
\[
 h_n(x_\pi)=1-a.
\]
The policy chooses $a$, while the rewarding action is $1-a$, so its reward on this state is zero. Because the state has positive probability,
\[
 J_n(\pi)\le1-\mu_n(1x_\pi)<1,
\]
contradicting optimality. Hence $B_\pi(n+1)>n^2$.

Only this one state was needed to establish the lower bound. Since $\pi$ has fixed runtime, the same bound applies on every other input of length $n+1$, including the easy states. Thus $\overline C_n(\pi)=B_\pi(n+1)>n^2$ for every $\pi\in\mathcal F_n$. The implementation $\pi_{\mathrm{fix}}$ makes this class nonempty. Combining the lower bound with the adaptive upper bound proves
\[
 \inf_{\pi\in\mathcal F_n}
 \frac{\overline C_n(\pi)}{\overline C_n(\pi_{\mathrm{ad}})}
 \ge\frac{n^2}{cn+1}.
\]
The reward rule is fixed throughout this argument: it always examines the program encoded in the observed state. Only the state $1x_\pi$ used to demonstrate failure depends on the candidate policy.
\end{proof}

\section{Additional Results}
\label{app:results}
\begin{figure}[h]
\centering
\includegraphics[width=\textwidth]{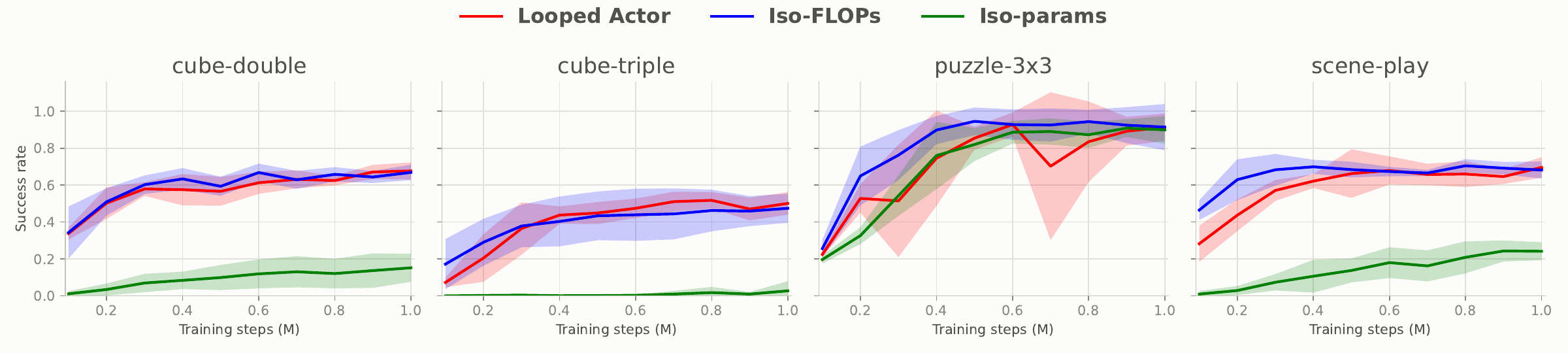}
\caption{Success rate of Looped Actor, Iso-FLOPs, and Iso-Parameters models on all four OGBench environments considered in this work during training (mean and standard deviation over 5 seeds).}
\label{fig:ogbench_learning_curves_nods_means}
\end{figure}

\end{document}

%% file: math_commands.tex
\usepackage{amsmath,amsfonts,bm}

\def\eqref#1{equation~\ref{#1}}
\def\1{\bm{1}}

\DeclareMathAlphabet{\mathsfit}{\encodingdefault}{\sfdefault}{m}{sl}
\SetMathAlphabet{\mathsfit}{bold}{\encodingdefault}{\sfdefault}{bx}{n}